\documentclass{article}
\usepackage[margin=1in]{geometry}
\usepackage[utf8]{inputenc}
\usepackage[colorlinks,linkcolor=black,citecolor=black,urlcolor=magenta,bookmarks=true]{hyperref}
\usepackage[dvipsnames]{xcolor}
\usepackage{subcaption}
\usepackage{multirow}

\usepackage{amsmath,amsthm,amssymb,dsfont}

\usepackage{nicefrac} 

\usepackage[conf={normal}]{proof-at-the-end}

\usepackage[utf8]{inputenc}
\usepackage[english]{babel}

\usepackage[backend=biber,maxbibnames=8,style=numeric,maxcitenames=2,backref=false,uniquelist=false,uniquename=false,sorting=none,noerroretextools=true,url=false,doi=false,isbn=false]{biblatex}
\usepackage{csquotes}
\renewbibmacro{in:}{%
	\ifentrytype{article}{}{%
		\printtext{\bibstring{in}\intitlepunct}}}
\renewbibmacro*{volume+number+eid}{%
	\printfield{volume}%
	\setunit*{\addnbspace}
	\printfield{number}%
	\setunit{\addcomma\space}%
	\printfield{eid}}
\renewbibmacro*{volume+number+eid}{%
	\setunit*{\addcomma\space}
	\printfield{volume}%
	\setunit*{\addcomma\space}
	\printfield{number}%
	\setunit{\addcomma\space}%
	\printfield{eid}
}
\DeclareFieldFormat[article]{volume}{\bibstring{volume}~#1}
\DeclareFieldFormat[article]{number}{\bibstring{number}~#1}

\newbibmacro{string+doiurlisbn}[1]{%
  \iffieldundef{doi}{%
    \iffieldundef{url}{%
      \iffieldundef{isbn}{%
        \iffieldundef{issn}{%
          #1%
        }{%
          \href{http://books.google.com/books?vid=ISSN\thefield{issn}}{#1}%
        }%
      }{%
        \href{http://books.google.com/books?vid=ISBN\thefield{isbn}}{#1}%
      }%
    }{%
      \href{\thefield{url}}{#1}%
    }%
  }{%
    \href{http://dx.doi.org/\thefield{doi}}{#1}%
  }%
}
\DeclareFieldFormat{title}{\usebibmacro{string+doiurlisbn}{\mkbibemph{#1}}}
\DeclareFieldFormat[article,incollection,unpublished,inproceedings]{title}%
{\usebibmacro{string+doiurlisbn}{\mkbibquote{#1}}}
\newcommand{\citep}{\parencite}
\newcommand{\citet}{\textcite}

\usepackage{cleveref}

\usepackage{booktabs}
\usepackage{enumitem}

\usepackage{xspace}
\usepackage{graphicx,multirow}
\graphicspath{{figs/}}

\usepackage{tikz}
\usetikzlibrary{matrix,arrows,shadows,positioning,fadings,calc,shapes,fit}
\usepackage{pgfplots}

\let\etoolboxforlistloop\forlistloop 
\let\forlistloop\etoolboxforlistloop 

\usepackage[linesnumbered,ruled,vlined]{algorithm2e}
\crefname{algocf}{algorithm}{algorithms}
\Crefname{algocf}{Algorithm}{Algorithms}

\usepackage{todonotes}

\newcommand{\Dc}{\mathcal{D}}
\newcommand{\Fc}{\mathcal{F}}
\newcommand{\Mc}{\mathcal{M}}
\newcommand{\Rc}{\mathcal{R}}
\newcommand{\Dds}{\mathds{D}}

\newcommand{\EE}{\mathds{E}}

\newcommand{\RR}{\mathds{R}}

\usepackage{thmtools, thm-restate}
\usepackage[normalem]{ulem}
\theoremstyle{definition}
\newtheorem{definition}{Definition}[section]
\newtheorem{theorem}{Theorem}[section]
\newtheorem{remark}{Remark}[section]

\pgfplotsset{compat=1.17} 

\newcommand{\dpvae}{DP\textsuperscript{2}-VAE\xspace}
\title{\dpvae: Differentially Private Pre-trained Variational Autoencoders}
\author{Dihong Jiang, Guojun Zhang, Mahdi Karami, Xi Chen, Yunfeng Shao, Yaoliang Yu}

\begin{document}

\maketitle

\begin{abstract}
    Modern machine learning systems achieve great success when trained on large datasets. However, these datasets usually contain sensitive information (e.g. medical records, face images), leading to serious privacy concerns. Differentially private generative models (DPGMs) emerge as a solution to circumvent such privacy concerns by generating privatized sensitive data. In this paper, we propose \dpvae, a novel training mechanism for variational autoencoders (VAE) with provable DP guarantees via \emph{pre-training on private data}. \dpvae is flexible and can be easily amenable to many other VAE variants. 
    We theoretically study the effect of pretraining on private data, and empirically verify its effectiveness on image dataset. 
\end{abstract}

\section{Introduction}
The success of modern machine learning (ML) algorithms and applications highly relies on the access to large-scale datasets \citep{DengDSLLL09,LewisYTL04,BennettL07}. However, there are increasing concerns on privacy leakage during the use of the data, especially the sensitive data (e.g. face images, medical records) that can be exploited by a malicious party, even though the original ML applications never intentionally do so. For example, \citet{FredriksonJR15} show that an attacker, who is only given a name and the white-box access to a face recognition model, can successfully recover face images of a particular person who appears in the training set (which is also known as the model inversion attack).

Prior efforts on developing privacy-preserving techniques include naive data anonymization \citep{NarayananS08}, $k$-anonymity \citep{Sweeney02}, $l$-diversity \citep{MachanavajjhalaKGV07}, $t$-closeness \citep{LiLV07}, semantic security \citep{GoldwasserM84}, information-theoretic privacy \citep{DiazWCS19}, and differential privacy (DP) \citep{Dwork06}, where the last one is recognized as a rigorous quantization of privacy, and becomes the gold-standard in current ML community. \citet{AbadiCGMMTZ16} propose DP-SGD algorithm, which then becomes the standard technique to train a DP learner. The core steps of DP-SGD are clipping gradient norm and injecting Gaussian noise to the gradient. 

Differentially private generative model (DPGM) aims to generate synthetic data that are distributionally similar to the private data while satisfying differential privacy guarantee, so that no one can infer private information from the generation. The major benefits of DPGMs are two-fold: (1) As a proxy for releasing private data; (2) Benefiting private data analysis tasks (e.g. data querying, ML tasks), i.e.~one can generate as much synthetic data as desired for data analysis tasks with DPGMs without incurring further privacy cost, as ensured by the post-processing theorem \citep{DworkR14}. 

Generative adversarial network (GAN) \citep{GoodfellowPMXWOCB14} attracts most attention in developing DPGMs \citep{XieLWWZ18, TorkzadehmahaniKP19, JordonYV19, LongWYKZGL21, AugensteinMRRKCMA20, ChenOF20}, while the related works based on variational autoencoder (VAE) \citep{KingmaWelling13} are relatively limited \citep{ChenXXLBKZ18, AcsMCD18, TakahashiTOK20}. Among related works, G-PATE \citep{LongWYKZGL21} first notes that in order to learn a DP generator, it is not necessary to make the discriminator DP, because only the generator will be released. This idea is also utilized by GS-WGAN \citep{ChenOF20}. We found that VAE is a natural model to be considered for further exploring this idea, because only the decoder of VAE needs to be released. Our additional motivation for considering VAE over GAN is \emph{two-fold}: ({\bf 1}) The minimax optimization of GAN leads to training instability \citep{MeschederGN18}, while VAE is easier to train. ({\bf 2}) VAE can estimate the joint density of input and latent variables, while GAN cannot.

The privacy-utility trade-off is one of the most important challenges in DP ML systems, i.e.~privacy is preserved at a cost of model utility. A line of recent works show that leveraging large public datasets (where there is little privacy leakage concern) as additional knowledge to pre-train a model which is then fine-tuned on private data with DP training algorithms can significantly improve the utility of a DP learner while keeping the same level of DP guarantee \citep{LuoWAL21, YuNBGHKKLMWYZ22, LiTLH22}. 

We remark that if the pre-training is conducted on private data, then the resulting model utility will be further improved, since the distribution shift between public and private data disappears \citep{WilesGSRKDC22}. However, DP fine-tuning a non-privately pretrained model in a straightforward manner may still violate DP guarantees (see \Cref{rmk:non-private}). How to pre-train on private data without breaching DP guarantee thus remains a tantalizing open question. 
In this work, we give a positive answer to this question by proposing a novel mechanism for training a differentially private pretrained (conditional) VAE (DP\textsuperscript{2}-VAE).

\section{Preliminary}
\label{sec:preliminary}
In this section, we recall background knowledge in differential privacy.

\subsection{Differential privacy}
Differential privacy is widely regarded as a rigorous quantization of privacy, which upper bounds the deviation in the output distribution of a randomized
algorithm given incremental deviation in input. Formally, we have the following definition:
\begin{definition}[$(\epsilon, \delta)$-DP \citep{DworkR14}]
A randomized mechanism $\mathcal{M}: \mathcal{D} \rightarrow \mathcal{R}$ with domain $\mathcal{D}$ and range $\mathcal{R}$ satisfies $(\epsilon,\delta)$-differential privacy if for any two adjacent inputs $d,d'\in\mathcal{D}$ and for any subset of
outputs $\mathcal{S}\subseteq\mathcal{R}$ it holds that
\begin{align}
    \mbox{Pr}[\mathcal{M}(d)\in\mathcal{S}]\leq \exp(\epsilon) \cdot \mbox{Pr}[\mathcal{M}(d')\in\mathcal{S}]+\delta
\end{align}
where adjacent inputs (a.k.a. neighbouring datasets) only differ in one entry. Particularly, when $\delta=0$, we say that $\mathcal{M}$ is $\epsilon$-DP.
\end{definition}

A famous theorem, i.e.~post-processing theorem, which is utilized by existing works (as well as ours) for proving DP guarantee of a published model, is given by: 
\begin{theorem}[Post-processing theorem, {\citep{DworkR14}}]
\label{theory:post}
If $\mathcal{M}$ satisfies $(\epsilon, \delta)$-DP, $F\circ \mathcal{M}$ will satisfy $(\epsilon, \delta)$-DP for any
function $F$ with $\circ$ denoting the composition operator.
\end{theorem}

\subsection{R{\'e}nyi differential privacy (RDP)}
\label{sec:RDP}
R{\'e}nyi differential privacy (RDP) extends ordinary DP using R{\'e}nyi's $\alpha$ divergence \citep{Renyi1961measures} and provides tighter and easier composition property than the ordinary DP notion. Formally, we recall
\begin{definition}[$(\alpha,\epsilon)$-RDP \citep{Mironov17}]
A randomised mechanism $\mathcal{M}$ is $(\alpha, \epsilon)$-RDP if for all adjacent inputs $D, D'$, R{\'e}nyi's $\alpha$-divergence (of order $\alpha > 1$) between the distribution of $\Mc(D,\texttt{AUX})$ and $\Mc(D',\texttt{AUX})$ satisfies:
\begin{align}
\Dds_{\alpha}(\Mc(D,\texttt{AUX}) \| \Mc(D',\texttt{AUX})) := \tfrac{1}{\alpha-1}\log \EE_{Z \sim Q} \left(\tfrac{P(Z)}{Q(Z)}\right)^\alpha \leq \epsilon,
\end{align}
where $P$ and $Q$ are the density of $\Mc(D,\texttt{AUX})$ and $\Mc(D',\texttt{AUX})$, respectively (w.r.t. some dominating measure $\mu$), and $\texttt{AUX}$ denotes auxiliary input (could be omitted if not applicable).
\end{definition}

Importantly, a mechanism satisfying $(\alpha,\epsilon)$-RDP also satisfies $(\epsilon+\frac{\log 1/\delta}{\alpha-1},\delta)$-DP for any $\delta\in (0,1)$.

Conveniently, RDP is linearly composable: 
\begin{theorem}[Composition of RDP \citep{Mironov17}]
\label{theory:compose}
If mechanism $\mathcal{M}_i$ satisfies $(\alpha,\epsilon_i)$-RDP  for $i=1,2,\ldots,k$, then releasing the composed mechanism $(\mathcal{M}_1,\ldots,\mathcal{M}_k)$ satisfies $(\alpha,\sum_{i=1}^k\epsilon_i)$-RDP.
\end{theorem}

We remark that $(\alpha,\infty)$-RDP (or $(\infty, \delta)$-DP) is not a rigorous notion to denote non-private mechanisms, because $\epsilon$ only tracks the upper bound of privacy loss.
We call a model $\Mc$ non-private if for any $\epsilon > 0$, there exist adjacent inputs $D, D'$ such that 
\begin{align}
\label{eq:def-np}
\Dds_{\alpha}(\Mc(D,\texttt{AUX}) \| \Mc(D',\texttt{AUX}) ) > \epsilon.
\end{align}
First, composing a fine-tuning mechanism (either DP or not) with a non-private model remains non-private.
Formally, we have:
\begin{restatable}{theorem}{nonprivate}
\label{theory:non-private}
Let $\Mc_{p}(D)$ denote a non-private (pre-trained) model, $\Fc$ denote any (fine-tuning) mechanism (differentially private or not).
Then, the composition $\Mc_{r}(D):= (\Mc_p(D), \Fc(D, \Mc_p(D)))$ remains non-private. 
\end{restatable}
See \Cref{sec:proof} for the proof. \Cref{theory:non-private} formalizes the intuition that releasing \emph{more} information (as in the composed mechanism $\Mc_r$, which releases both the output of $\Fc$ and $\Mc_p$) can only make one's mechanism \emph{less} (differentially) private, while \Cref{theory:compose} shows that it degrades the privacy guarantee at most linearly.

However, what we are actually interested in is whether (DP) fine-tuning a non-private model violates DP guarantee, i.e.~releasing $\Fc(D,\Mc_p(D))$ instead of releasing output of both $\Fc$ and $\Mc_p$ in the composition is DP or not. Here we remark that DP fine-tuning a non-private model may still be non-private:
\begin{remark}
\label{rmk:non-private}
Let $\Fc$ denote a DP mechanism, and $\Mc_p(D)$ denote a non-private model. $\Fc(D,\Mc_p(D))$ can still be non-private.
\end{remark}
We can construct two examples to illustrate the remark. 
\begin{itemize}[leftmargin=*]
    \item Consider a DP $\Fc$ as $\Fc(D,\texttt{AUX})=c$, where $c$ is a constant. Then $\Fc(D, \Mc_p(D))=c$ is  private.
    \item Consider a DP $\Fc$ as $\Fc(D,\texttt{AUX})=\texttt{AUX}$. Then $\Fc(D, \Mc_p(D))=\Mc_p(D)$ is non-private.
\end{itemize}

Despite the pessimistic result in \Cref{rmk:non-private}, our method, which also utilizes the idea of pre-training on private input, can \emph{circumvent} privacy leakage by explicitly decomposing pre-trained model into two halves, which will be explained in the end of \Cref{sec:method}.

In our work, we also adopt the Gaussian mechanism for achieving RDP:
\begin{definition}[Gaussian mechanism for RDP \citep{DworkR14,Mironov17}]
\label{theory:gaussian}
Let $f:\Dc \to \RR^p$ be an arbitrary $p$-dimensional function with sensitivity:
\begin{align}
    \Delta_2 f = \max_{D, D'}\|f(D)-f(D')\|_2
\end{align}
for all adjacent datasets $D, D'\in\Dc$. The Gaussian mechanism $\Mc_\sigma$ perturb the output of $f$ with Gaussian noise:
\begin{align}
\Mc_\sigma = f(D) + \mathcal{N}(0, \sigma^2\cdot \mathbb{I})
\end{align}
where $\mathbb{I}$ is identity matrix. Then, $\Mc_\sigma$ satisfies $(\alpha,\frac{\alpha(\Delta_2 f)^2}{2\sigma^2})$-RDP.
\end{definition}

\subsection{Variational autoencoder (VAE)}
Let $x\in \mathcal{X}$ denote data and $z\in\mathcal{Z}$ denote latent variable. VAE consists of two components: an encoder $T_\theta:\mathcal{X}\rightarrow\mathcal{Z}$, where $z=T_\theta(x)\sim q_\theta(z|x)$ ($q_\theta(z|x)$ is known as a variational inference to approximate the intractable true posterior $p(z|x)$), and a decoder $S_\phi:\mathcal{Z}\rightarrow\mathcal{X}$, where $x=S_\phi(z)\sim p_\phi(x|z)$ . Given a tractable prior $p(z)$, e.g. Gaussian, we can rewrite $\log p(x)$ as:
\begin{align}
    \log p(x) &= \mathbb{E}_{q_\theta(z|x)}[\log p_\phi(x|z)] - \mathbb{D}_{KL}(q_\theta(z|x)||p(z)) + \mathbb{D}_{KL}(q_\theta(z|x)||p(z|x)) \\
    &\geq \mathbb{E}_{q_\theta(z|x)}[\log p_\phi(x|z)] - \mathbb{D}_{KL}(q_\theta(z|x)||p(z)) := \textrm{Evidence lower bound (ELBO)}
\end{align}

where the inequality holds due to the non-negativity of Kullback–Leibler (KL) divergence. Therefore, the training of VAE proceeds by maximizing $\log p(x)$ via maximizing the tractable ELBO. A simple extension to conditional generation is to encode label information into the input.

\section{Method: DP\textsuperscript{2}-VAE}
\label{sec:method}
Our idea is inspired by GS-WGAN, where the authors warm-start (i.e.~pre-train) discriminators along with a non-private generator to bootstrap the training process, and then privately train the generator while continuing normally training the pretrained discriminators, to retain differential privacy for the generator. The rationale behind it is the fact that only the generator will be released after completing the training of a GAN, so the discriminator can be non-private. We adapt this idea to VAE, where only the decoder will be released, thus it is not necessary to make the encoder private. Prior works show that subsampling can improve privacy \citep{WangBK19,BalleBG18}, so we subsample the whole training set into different subsets. Our method contains two stages.  Each encoder is pretrained with a new decoder on each subset in stage 1. Proceeding to stage 2, we first reinitialize the decoder. In each training iteration, we randomly query a pre-trained encoder and its associated subsampled dataset, then privately train the decoder and normally train the encoder. 
Specifically, a dataset $D$ is randomly shuffled and subsampled (without replacement) into subsets $D_k$ (for $k=1,2,\ldots,K$, we use $K=2500$ in this work), then the training of DP\textsuperscript{2}-VAE can be summarized into two main stages:
\begin{itemize}[leftmargin=*]
    \item {\bf Stage 1:} we reinitialize a decoder $S_\phi$, then pre-train both encoder $T_{\theta_k}$  and decoder $S_\phi$ on $D_k$, and save $T_{\theta_k}$ at the end of pre-training (for $k=1,2,\ldots,K$). 
    \item {\bf Stage 2:} we reinitialize a decoder $S_\phi$. In each training iteration, we randomly query a pre-trained encoder $T_{\theta_i}$ and associated $D_i$, then update parameters of $S_\phi$ and $T_{\theta_i}$ by private and normal training algorithms, respectively.
\end{itemize}
\subsection{Stage 1: Pre-training encoders on private input}
Stage 1 is similar to normally training a conditional VAE with gradient clipping. Differently, we partition the dataset into subsets, and each encoder is pre-trained on a subset with a reinitialized decoder. The detailed algorithm is given in \Cref{algo:s1}. 
The weights of pre-trained encoder in stage 1 will be transferred as input to stage 2. Note that encoders are independent of each other, so the pre-training can be conducted in parallel.

\begin{algorithm}[t]
\DontPrintSemicolon
\KwIn{Private training set $D=\{(X,y)\in \mathcal{X}\times\mathcal{Y}\}^N$, ELBO of conditional VAE $\mathcal{L}(X,y;\cdot,\cdot)$, the number of pre-training iterations $T_p$, the number of encoders $K$, learning rate $\eta_p$, gradient clipping bound $C$, batch size $B$, Adam optimizer $Adam(\cdot;\texttt{aux})$}

Subsample $D$ into $K$ subsets $D_1,D_2,\ldots,D_K$

\For{$k \gets 1$ \textbf{to} $K$}{
    Randomly initialize $\theta_k,\phi$ 
    \tcp*{Initialize encoder and decoder}
    \For{$t \gets 1$ \textbf{to} $T_p$}{
        Sample a batch $(X_b,y_b)=\{(X_i,y_i)\}_{i=1}^B\subseteq D_k$
        
        $g_\phi(X_b,y_b)=\nabla_\phi\mathcal{L}(X_b,y_b;\theta_k,\phi)$
        \tcp*{Compute gradient of decoder}
        
        $g_\phi(X_b,y_b)=g_\phi(X_b,y_b)/\max(1,\frac{\|g_\phi(X_b,y_b)\|_2}{C})$
        \tcp*{Clip the gradient}
        
        $g_{\theta_k}(X_b,y_b)=\nabla_{\theta_k}\mathcal{L}(X_b,y_b;\theta_k,\phi)$
        \tcp*{Compute gradient of encoder}
        
        $g_{\theta_k}(X_b,y_b)=g_{\theta_k}(X_b,y_b)/\max(1,\frac{\|g_{\theta_k}(X_b,y_b)\|_2}{C})$
        \tcp*{Clip the gradient}
        
        $\phi=\phi-\eta_p*Adam(g_\phi(X_b,y_b);\texttt{aux})$
        \tcp*{Update parameters of decoder}
        
        $\theta_k=\theta_k-\eta_p*Adam(g_{\theta_k}(X_b,y_b);\texttt{aux})$
        \tcp*{Update parameters of encoder}
    }
}
\KwOut{$\theta_1,\theta_2,\ldots,\theta_K$} 
\caption{DP\textsuperscript{2}-VAE: Stage 1}
\label{algo:s1}
\end{algorithm}
\begin{algorithm}[t]
\DontPrintSemicolon
\KwIn{Private training set $D=\{(X,y)\in \mathcal{X}\times\mathcal{Y}\}^N$, ELBO of conditional VAE $\mathcal{L}(X,y;\cdot,\cdot)$, number of training iterations $T$, learning rate $\eta$, batch size $B$, noise multiplier $\sigma$, gradient clipping bound $C$, the number of pre-trained encoders $K$, pre-trained encoder $\theta_k$ ($k=1,2,\ldots,K$), Adam optimizer $Adam(\cdot;\texttt{aux})$, DP constraint  $\delta=10^{-5}$, DP constraint epsilon calculator based on RDP accountant $Eps$}

Subsample $D$ into $K$ subsets $D_1,D_2,\ldots,D_K$  \tcp*{The same as in stage 1}
Randomly initialize $\phi$ 
\tcp*{Initialize decoder $S_\phi$}

\For{$k \gets 1$ \textbf{to} $K$}{
    Load pretrained encoder $\theta_k$
}
\For{$t \gets 1$ \textbf{to} $T$}{
    Randomly query an index $k\sim\texttt{Unif}(1,K)$ 
    
    Sample a random batch $\{(X_i,y_i)\}_{i=1}^B$ from $D_k$

    \For{$i \gets 1$ \textbf{to} $B$}{
        $g_\phi(X_i,y_i)=\nabla_\phi\mathcal{L}(X_i,y_i;\theta_k,\phi)$
        \tcp*{Compute gradient of decoder}

        $g_{\theta_k}(X_i,y_i)=\nabla_{\theta_k}\mathcal{L}(X_i,y_i;\theta_k,\phi)$
        \tcp*{Compute gradient of encoder}

        $\theta_k=\theta_k-\eta*Adam(g_{\theta_k}(X_i,y_i);\texttt{aux})$
        \tcp*{Update parameters of encoder}
    }

    $\bar{g}_\phi=\frac1B\sum_{i=1}^B g_\phi(X_i,y_i)$
    
    $\bar{g}_\phi = \bar{g}_\phi \cdot \min\left(1,\frac{C}{\|\bar{g}_\phi\|_2}\right)$
    \tcp*{Clip the gradient}
    
    $\tilde{g}_\phi = \bar{g}_\phi + \mathcal{N}(0, \sigma^2C^2\mathbb{I}))$

    $\phi=\phi-\eta*Adam(\tilde{g}_\phi;\texttt{aux})$
    \tcp*{Update parameters of decoder}
    
}

$\epsilon=Eps(K,\sigma,T,\delta)$
\tcp*{Calculate DP constraint $\epsilon$}

\KwOut{$\phi,\epsilon$}
\caption{DP\textsuperscript{2}-VAE: Stage 2}
\label{algo:s2}
\end{algorithm}

\subsection{Stage 2: Privately training the decoder with pre-trained encoder}
In stage 2, we load pre-trained encoders $T_{\theta_k}$ ($k=1,2,\ldots,K$) obtained in stage 1. In each training iteration, we randomly query one encoder and its associated training subset, then update decoder and encoder on the subset by private and non-private training algorithms, respectively, as described in \Cref{algo:s2}.

We note that an alternative to stage 2 is to fix the pre-trained encoder. However, we empirically found that keep training the pre-trained encoder outperforms the aforementioned alternative, thus we adopt the strategy as described in this subsection in our work.

\begin{restatable}{theorem}{main}
\label{theory:main}
Each update step in the decoder $S_\phi$ in stage 2 satisfies $(\alpha, \frac{2\alpha}{\sigma^2})$-RDP.
\end{restatable}
We defer the proof to \Cref{sec:proof}.

While \Cref{rmk:non-private} reveals that privacy cannot be reliably protected by trivially DP fine-tuning a non-private pre-trained model, it does not apply to DP\textsuperscript{2}-VAE, even though we utilize a similar idea of pre-training on private data. The reason lies in the fact that we explicitly decompose the VAE into two halves, where only the pre-trained encoders are loaded. At the beginning of stage 2, the decoder $S_\phi$ is randomly initialized, which eliminates private information in the pre-trained decoder, i.e.~$\phi^{(0)}$ is $(\alpha,0)$-RDP. By \Cref{theory:main}, each decoder update step is DP, thus the released decoder $S_\phi$ as a composition of DP mechanisms is differentially private. We provide a more intuitive interpretation to further illustrate this point: in stage 2, we first randomly initialize the decoder, then perturb the gradient with Gaussian noise when training the decoder, such that the information flow in the decoder is always privatized. The schematic of DP\textsuperscript{2}-VAE is given in \Cref{sec:schematic}.

\section{Experiments}

In this section, we evaluate and compare DP\textsuperscript{2}-VAE against SoTA baselines on MNIST\citep{LecunBBH98}. Implementation details are given in \Cref{sec:implement}.
\subsection{Experimental setup}
\paragraph{Evaluation tasks \& metrics:} Since privacy-utility trade-off is the main concern in DP learners, we consider the following two tasks for extensive quantitative evaluations given the same set of privacy parameters (i.e.~same $(\epsilon, \delta)$-DP) via 60k generated images:
\begin{itemize}[leftmargin=*]
    \item Generation quality, which is measured by Fr{\'e}chet Inception Distance (FID) \citep{HeuselRUNH17}. 
    \item Model utility. We train three different classifiers, e.g. logistic regression (LR), multi-layer perceptron (MLP), and convolutional neural network (CNN), on generated images, then test the classifier on real images, where the performance is measured by classification accuracy. We take 5 runs and report the average.
    
\end{itemize}

\paragraph{SoTA baselines:} Our method is compared with following baseline methods that are also developed on image datasets, i.e.~DP-CGAN \citep{TorkzadehmahaniKP19}, DP-MERF \citep{HarderAP21}, Datalens \citep{WangWLRZL21}, PATE-GAN \citep{JordonYV19}, G-PATE \citep{LongWYKZGL21}, GS-WGAN \citep{ChenOF20}, DP-Sinkhorn \citep{CaoBVFK21}. For more details, we refer interesting readers to \Cref{sec:related_work} and respective references.

\subsection{Comparison with SoTA baselines}
The qualitative visualization comparison is shown in \Cref{fig:qualitative_mnist_eps10}, and quantitative comparison is given in \Cref{tbl:mnist_fmnist_eps10}. Quantitatively, \Cref{tbl:mnist_fmnist_eps10} indicates that our method achieves comparable performance in classification accuracy.

\begin{figure}
\centering
    \includegraphics[width=0.4\linewidth]{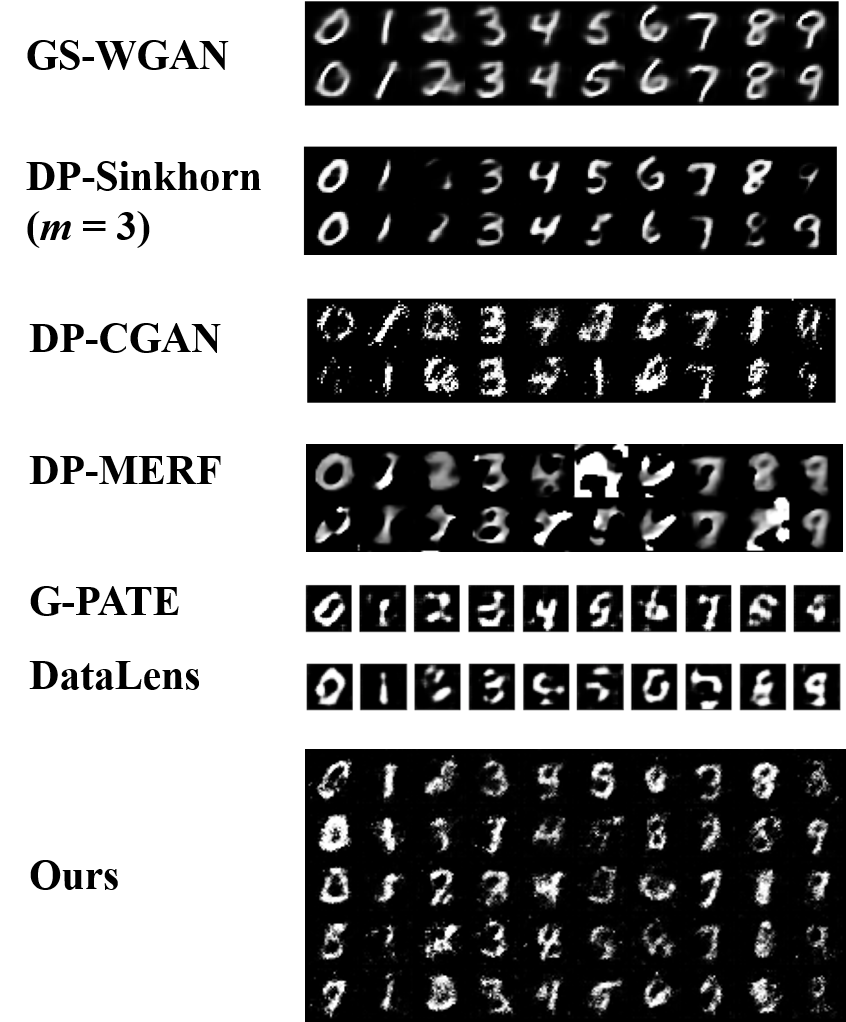}
    \caption{Qualitative comparison on MNIST and Fashion MNIST under $(10, 10^{-5})$-DP. Images of DP-CGAN, GS-WGAN, DP-Sinkhorn are cited from \citep{CaoBVFK21}. Images of G-PATE and DataLens are cited from their papers, respectively.}
    \label{fig:qualitative_mnist_eps10}
    \vspace{-1em}
\end{figure}

\begin{table*}[t]
\caption{Quantitative comparison on MNIST given $(10, 10^{-5})$-DP. Acc denotes classification accuracy, which is shown in \%. $\uparrow$ and $\downarrow$ refer to higher is better or lower is better, respectively. Results of DP-CGAN, GS-WGAN, DP-Sinkhorn are cited from \citep{CaoBVFK21}. Results of G-PATE and DataLens are cited from their papers, respectively.}
\centering
\label{tbl:mnist_fmnist_eps10}
\scalebox{0.95}{
\begin{tabular}{llcccc}
\toprule \multirow{2}{*}{Method}  & \multirow{2}{*}{$\epsilon$} &
\multicolumn{4}{c}{MNIST} \\
\cmidrule(lr){3-6}
          & &  FID $\downarrow$ & LR & MLP & CNN \\
          & &    & Acc $\uparrow$ & Acc $\uparrow$ & Acc $\uparrow$ \\
\midrule
Real data & $\infty$ & 1.6 & 92.2 & 97.5 & 99.3 \\
\midrule
DP-CGAN & 10 & 179.2 & 60 & 60 & 63 \\
DP-MERF & 10 & 121.4 & 79.1 & 81.1 & 82.0 \\
G-PATE & 10 & 150.6 & N/A & N/A & 80.9 \\
DataLens & 10 & 173.5 & N/A & N/A & 80.7 \\
GS-WGAN & 10 & 61.3 & 79 & 79 & 80 \\
DP-Sinkhorn ($m=1$) & 10 & 61.2 & 79.5 & 80.2 & 83.2\\
DP-Sinkhorn ($m=3$) & 10 &  55.6 & 79.1 & 79.2 & 79.1\\
\midrule
Ours & 10 & 134.3 &  78.4 & 77.8 & 81.2\\
\bottomrule
\end{tabular}
}
\end{table*}

\section{Related work} 
\label{sec:related_work}
We group related work by different categories of generative models:
\paragraph{GAN:} The vast majority of related works are based on GAN. DP-GAN \citep{XieLWWZ18} first trains GAN with DP-SGD algorithm, where the discriminator is trained with DP-SGD, then the generator is automatically DP as ensured by post-processing theorem. DP-CGAN \citep{TorkzadehmahaniKP19} extends DP-GAN into conditional generative setting. Private Aggregation of Teacher Ensembles (PATE) \citep{PapernotAEGT17,PapernotSMRTE18} is a different mechanism for learning a DP model, and a few related works tried to apply PATE to GAN. PATE-GAN \citep{JordonYV19} trains $k$ teacher discriminators on $k$ disjoint partitioned datasets, and the label is predicted by aggregating teacher votes that are perturbed with Laplace noise, so that the discriminator is DP. The PATE mechanism makes the discriminator non-differentiable, thus a student discriminator is trained with teacher ensembles, which can be used to train the generator. G-PATE \citep{LongWYKZGL21} is another work extending GAN with PATE. The authors first observed that instead of learning a DP discriminator, it suffices to ensure the information flow from the discriminator to the generator is private to make the generator DP, i.e.~sanitizing the aggregated gradients from teacher discriminators to the generator. However, gradient vectors need to be discretized in each dimension to employ the PATE mechanism that only takes categorical data as input. DataLens \citep{WangWLRZL21} further improves G-PATE by introducing a three-step gradient compression and aggregation algorithm called TopAgg. GS-WGAN \citep{ChenOF20} explores the gradient sanitization idea from G-PATE, and applies it to training Wasserstein GAN (WGAN) with DP-SGD algorithm, so that no discretization is required.
\paragraph{VAE:}  DP-VaeGM \citep{ChenXXLBKZ18} trains $k$ VAEs on $k$ classes of private data with DP-SGD algorithm, and return the union as generation. This work only evaluates their model against various privacy attacks. DP-kVAE \citep{AcsMCD18} first partitions the dataset into $k$ clusters by differentially private kernel $k$-means method, then trains $k$ VAEs on each data cluster with DP-SGD. PrivVAE designs a term-wise DP-SGD that restricts the gradient sensitivity at $O(1)$, because the authors observe that when additional divergence is added to the training objective of VAE as a regularization term, the gradient sensitivity will increase from $O(1)$ to $O(B)$ (where $B$ is the batch size), which is not applicable to our work since we use vanilla (conditional) VAE. It is worth mentioning that both DP-VaeGM and DP-kVAE essentially directly training VAE with DP-SGD algorithm, so we think there is potential to improve DP-VAEs. 

\paragraph{Others:} DP-NF \citep{WaitesCummings21} directly trains a flow-based model by DP-SGD algorithm. DP-MERF \citep{HarderAP21} proposes to perturb embeddings (random Fourier features) of input with Gaussian noise, then training a generator by minimizing the maximum mean discrepancy (MMD) between the noisy embedding of private input and embedding of generation. DP-Sinkhorn \citep{CaoBVFK21} proposes to train a DP generator by minimizing the optimal transport distance between real and generated distribution with DP-SGD algorithm. There are also some DPGMs developed from graphical models, such as PrivBayes \citep{ZhangCPSX17} and PrivSyn \citep{ZhangWLHBHCZ21}, where the idea is to use a selected set of low-degree marginals to represent a dataset (mainly low dimensional dataset such as tabular datasets), then synthesizing data from noise-perturbed marginals. However, it cannot scale well on high dimensional image datasets because the number of marginals will exponentially increase to sufficiently represent an image dataset.

\section{Conclusion}
In this paper, we propose DP\textsuperscript{2}-VAE, a novel mechanism for training a differentially private (conditional) VAE on high-dimensional data. By exploring the insight that only the decoder of a VAE will be published, both pretraining encoders in stage 1 and training encoders in stage 2 can be non-private, while we only need to privately train the decoder, such that the noise perturbation in the private training is minimized. DP\textsuperscript{2}-VAE can be readily extended to other variants of VAE, which is expected to benefit practical deployment. We demonstrate the effectiveness of DP\textsuperscript{2}-VAE by comparing with a wide range of SoTA baselines.

\section*{Acknowledgement}
We thank Kiarash Shaloudegi and Saber Malekmohammadi for early discussion. It is worth mentioning that the privacy analysis in the prior version\footnote{\url{https://arxiv.org/abs/2208.03409v1}} was incorrect. We thank Alex Bie and Aaron Roth for pointing out this issue.

\clearpage

\printbibliography[title={References}]

\clearpage
\appendix
\section{Proof}
\label{sec:proof}
\nonprivate*
\begin{proof}
We adopt RDP in this proof. 
Let $\Mc_{p}:\Dc\to\Rc_1$ be a non-private (pre-trained) model (see \eqref{eq:def-np}), 
and $\Fc:\Dc \times\Rc_1\to\Rc_2$ be any mechanism.
We show that the composed mechanism 
$\Mc_{r}: \Dc \to \Rc_1\times\Rc_2, D\mapsto \big(\Mc_p(D), \Fc(D, \Mc_{p}(D))\big)$
remains non-private. 

Indeed, fix any $\epsilon > 0$ and choose $D, D'$ such that 
\begin{align}
\Dds_\alpha( \Mc_p(D) \| \Mc_p(D') ) > \epsilon,
\end{align}
which is possible due to $\Mc_p$ being non-private.

Let $P$ and $Q$ denote the density of $\Mc_p(D)$ and $\Mc_p(D')$, respectively. 
Applying the decomposition rule (see the proof of Proposition 1 in \cite{Mironov17}), we obtain 
\begin{align}
\Dds_\alpha\big(\Mc_{r}(D) \|\Mc_{r}(D') \big) 
&= \tfrac{1}{\alpha-1} \log \EE_{Z \sim Q} \bigg[ \Big(\tfrac{P(Z)}{Q(Z)}\Big)^{\alpha} \underbrace{\exp\Big( (\alpha-1) \Dds_\alpha \big(\Fc(D, Z) \|\Fc(D', Z) \big) \Big)}_{\geq 1} \bigg] 
\\
&\geq \tfrac{1}{\alpha-1} \log \EE_{Z \sim Q} \left[ \Big(\tfrac{P(Z)}{Q(Z)}\Big)^{\alpha} \right] 
= \Dds_\alpha( \Mc_p(D) \| \Mc_p(D') ) > \epsilon,
\end{align}
where the inequality follows from the fact that R{\'e}nyi's $\alpha$-divergence is always nonnegative (when $\alpha > 1$). 

Since $\epsilon$ is arbitrary, we have proved that $\Mc_r$ remains non-private.
\end{proof}

\main*
\begin{proof}
Let $g_\phi$ be the gradient function of decoder $S_\phi$. Consider two adjacent batches $D,D'$ of size $B$.
Since $\bar{g_\phi}(D)=\frac1B\sum_{i=1}^B g_\phi(X_i,y_i)$ and $\|\bar{g_\phi}(D)\|_2 \leq C$, we know:
\begin{align}
    \Delta_2 \bar{g_\phi} = \max_{D,D'}\|\bar{g_\phi}(D)-\bar{g_\phi}(D')\|_2  \leq  2C.
\end{align}
As $\tilde{g}_\phi = \bar{g}_\phi + \mathcal{N}(0, \sigma^2C^2\mathbb{I}))$ in line 17, by Gaussian mechanism (\Cref{theory:gaussian}) we know each update step in the decoder (releasing $\tilde{g_\phi}$) satisfies $(\alpha, \frac{\alpha(\Delta_2 \bar{g_\phi})^2}{2(\sigma C)^2})$-RDP, i.e.~$(\alpha, \frac{2\alpha}{\sigma^2})$-RDP.
\end{proof}

\section{Implementation}
\label{sec:implement}
\subsection{Architecture \& hyperparameters} Our conditional VAE code is adapted from a \href{https://github.com/AntixK/PyTorch-VAE/blob/master/models/cvae.py}{public repo}, where the architecture sequentially contains input layer, encoder, linear layers (for mean and variance, respectively),  decoder input layer, decoder, and a final layer. The variation mainly lies in the number of hidden units and the number of convolutional layers in both encoder and decoder, as well as the number of latent dimensions. For MNIST and Fashion MNIST, we use two convolutional layers with 512 and 256 hidden units in both encoder and decoder, along with 8 latent dimensions. For CelebA, we use three convolutional layers with 512, 256 and 128 hidden units in both encode and decoder, along with 16 latent dimensions.

\subsection{Privacy implementation} We use a public repo, i.e.~\href{https://github.com/ChrisWaites/pyvacy}{pyvacy}, for implementing DP training algorithm and epsilon calculation. Pyvacy tracks the privacy loss by RDP accountant, which is a PyTorch implementation based on \href{https://github.com/tensorflow/privacy}{Tensorflow Privacy}.

\subsection{Fr{\'e}chet Inception Distance (FID)} FID calculates the distance between the feature vectors extracted by InceptionV3 pool3 layer \citep{SzegedyVISW16} on real and synthetic samples. Specifically,
\begin{align}
    \textrm{FID} = \|\mu_r-\mu_g\|_2^2 + \textrm{Tr}(\Sigma_r+\Sigma_g-2(\Sigma_r\Sigma_g)^{\frac12})
\end{align}
where $X_r\sim \mathcal{N}(\mu_r,\Sigma_r)$ and $X_g \sim \mathcal{N}(\mu_g,\Sigma_g)$ are activations of InceptionV3 pool3 layer of real images and generated images, respectively, and Tr($A$) refers to the trace of a matrix $A$. Intuitively, a lower FID means the generation $X_g$ is more realistic (or more similar to $X_r$). We use a  \href{https://github.com/mseitzer/pytorch-fid}{PyTorch implementation} for computing FID, which will resize images and repeat channels three times for grayscale images to meet the input size requirement.  

\subsection{Classification task} We follow \citet{CaoBVFK21} for the classifier implementation. We import scikit-learn package for implementation logistic regression classifier (e.g. from sklearn.linear\_model import LogisticRegression) with default parameter settings. 

The MLP network consists of following layers: linear($input\_dim, 100$) $\rightarrow$ ReLU $\rightarrow$ linear(100, $output\_dim$) $\rightarrow$ Softmax. 

The CNN consists of following layers: Conv2d($input\_channels$, 32, kernel\_size=3, stride = 2, padding=1) $\rightarrow$ Dropout(p=0.5) $\rightarrow$ ReLU $\rightarrow$ Conv2d(32, 64, kernel\_size=3, stride = 2, padding=1) $\rightarrow$ Dropout(p=0.5) $\rightarrow$ ReLU $\rightarrow$ flatten $\rightarrow$ linear($flatten\_dim,output\_dim$) $\rightarrow$ Softmax.

Both MLP and CNN are optimized by Adam with default parameters. All classifiers are trained on synthetic data, and we report test accuracy on real test data as the evaluation metric.

\section{Framework}
\label{sec:schematic}
The schematic of DP\textsuperscript{2}-VAE is depicted in \Cref{fig:framework}. In stage 1, we normally pre-train each encoder with a reinitialized decoder. In stage 2, we only transfer the pre-trained weights of encoders, and reinitialize the decoder in the beginning, then we train the decoder from scratch by private training algorithm while normally updating encoders.

\begin{figure}
    \centering
    \includegraphics[width=0.7\linewidth]{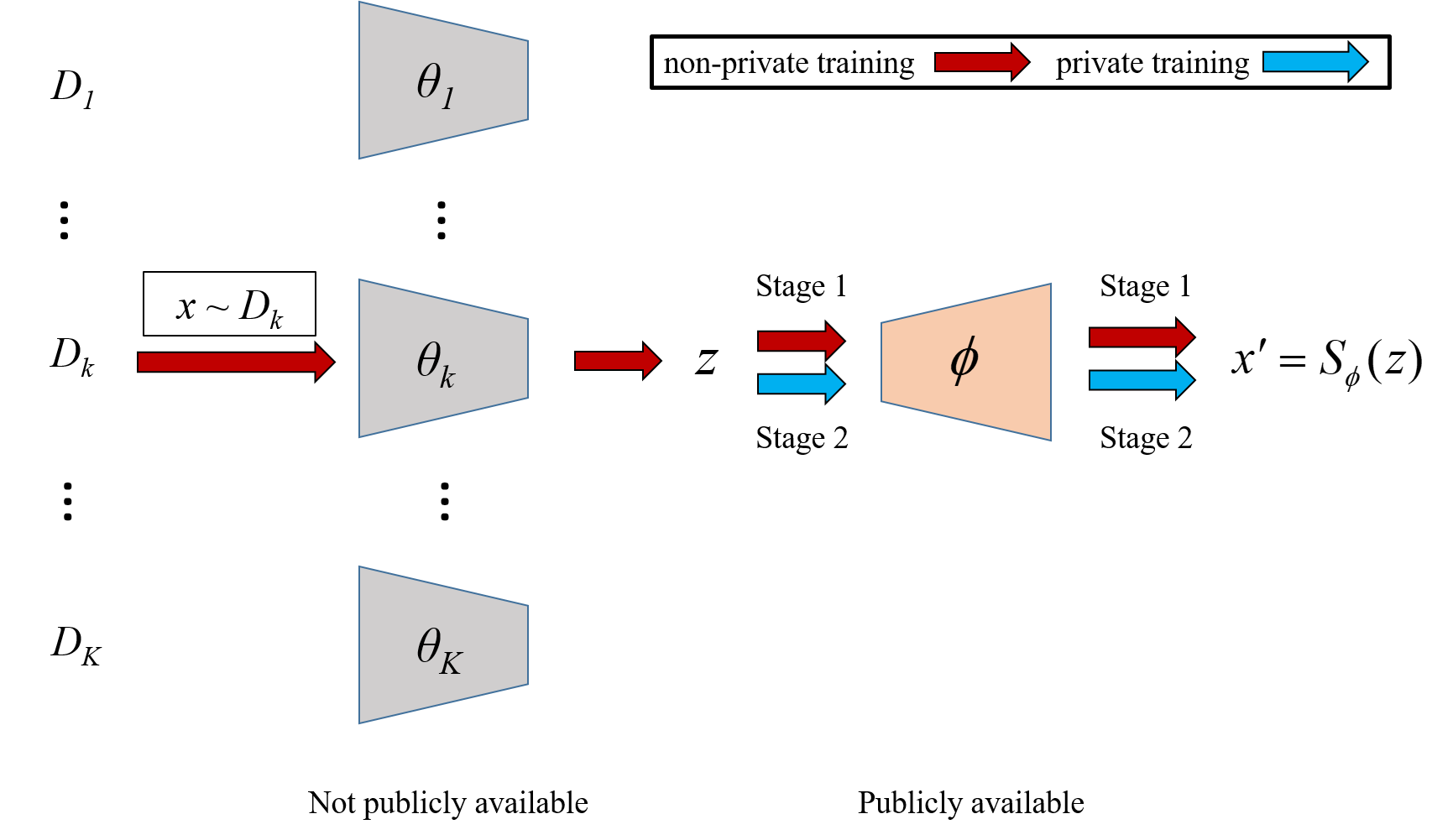}
    \caption{The two-stage strategy of training a DP\textsuperscript{2}-VAE. In stage 1, we (non-privately) pre-train encoder $T_{\theta_i}$ along with a new decoder $S_\phi$ on private input $D_i$ (for $i=1,2,\ldots,K$). In stage 2, we only load pre-trained encoders, and reinitialize the decoder in the beginning, then we train the $S_\phi$ from scratch with private training algorithm while keeping updating encoders normally.}
    \label{fig:framework}
\end{figure}

\end{document}